%% file: egpaper_for_review.tex
\documentclass[10pt,twocolumn,letterpaper]{article}

\usepackage{iccv}
\usepackage{times}
\usepackage{epsfig}
\usepackage{graphicx}
\usepackage{amsmath}
\usepackage{amssymb}

\usepackage[utf8]{inputenc} 
\usepackage[T1]{fontenc}    
\usepackage{url}            
\usepackage{booktabs}       
\usepackage{amsfonts}       
\usepackage{nicefrac}       
\usepackage{microtype}      
\usepackage{graphicx}
\usepackage{amsthm}
\newtheorem{thm}{Theorem}[section]
\newtheorem{lemma}[thm]{Lemma}
\newcommand{\tabincell}[2]{\begin{tabular}{@{}#1@{}}#2\end{tabular}}  

\usepackage[pagebackref=true,breaklinks=true,letterpaper=true,colorlinks,bookmarks=false]{hyperref}

\iccvfinalcopy 

\def\iccvPaperID{3133} 

\ificcvfinal\fi

\begin{document}

\title{DTNN: Energy-efficient Inference with Dendrite Tree Inspired Neural Networks for Edge Vision Applications}

\author{Tao Luo\textsuperscript{1}, Wai Teng Tang\textsuperscript{2}, Matthew Kay Fei Lee\textsuperscript{3}, Chuping Qu\textsuperscript{1}, Weng-Fai Wong\textsuperscript{4}, Rick Goh\textsuperscript{1}\\
Institute of High Performance Computing\textsuperscript{1}, Grab Holdings Inc.\textsuperscript{2}, \\ Stanford University\textsuperscript{3}, National University of Singapore\textsuperscript{4}\\
{\tt\small \{luo\_tao, qu\_chuping, gohsm\}@ihpc.a-star.edu.sg; waiteng.tang@grabtaxi.com;} \\
{\tt\small mattlkf@outlook.com; wongwf@nus.edu.sg}

}

\maketitle
\ificcvfinal\thispagestyle{empty}\fi

\begin{abstract}
Deep neural networks (DNN) have achieved remarkable success in computer vision (CV). However, training and inference of DNN models are both memory and computation intensive, incurring significant overhead in terms of energy consumption and silicon area. In particular, inference is much more cost-sensitive than training because training can be done offline with powerful platforms, while inference may have to be done on battery powered devices with constrained form factors, especially for mobile or edge vision applications.
In order to accelerate DNN inference, model quantization was proposed. However previous works only focus on the quantization rate without considering the efficiency of operations. In this paper, we propose {\bf D}endrite-{\bf T}ree based {\bf N}eural {\bf N}etwork ({\bf DTNN}) for energy-efficient inference with table lookup operations enabled by activation quantization. In DTNN both costly weight access and arithmetic computations are eliminated for inference. We conducted experiments on various kinds of DNN models such as LeNet-5, MobileNet, VGG, and ResNet with different datasets, including MNIST, Cifar10/Cifar100, SVHN, and ImageNet.
DTNN achieved significant energy saving (19.4$\times$ and 64.9$\times$ improvement on ResNet-18 and VGG-11 with ImageNet, respectively) with negligible loss of accuracy.
To further validate the effectiveness of DTNN and compare with state-of-the-art low energy implementation for edge vision, we design and implement DTNN based MLP image classifiers using off-the-shelf FPGAs.
The results show that DTNN on the FPGA, with higher accuracy, could achieve orders of magnitude better energy consumption and latency compared with the state-of-the-art low energy approaches reported that use ASIC chips.
\end{abstract}

\section{Introduction}
DNN outperforms many other methods in machine learning because of its human level performance in various kinds of application fields~\cite{ouyang2013joint,russakovsky2015imagenet,amodei2016deep}.
In particular, {\em convolutional neural networks} (CNN) have achieved state-of-the-art results on many computer version problems, including object detection and image classification~\cite{krizhevsky2012imagenet,ren2015faster}.

In general, computations on DNN model consist of two parts, comprising of training NN model and inference using the trained model.
However, both of the two parts are memory and computation intensive especially for CNN with deep depth.
In order to facilitate NN computation, powerful Graphic Processing Units (GPUs) are adopted to train the DNN model~\cite{coates2013deep}.
However, GPUs are not suitable for inference especially when deployed in embedded systems to cope with real-world applications such as mobile applications and Internet of Things (IoT), which requires both low power and real-time speed~\cite{stromatias2015scalable,lin2017towards,hubara2016binarized,Esser11441}.

In order to accelerate DNN for efficient inference, various kinds of techniques are proposed. This includes model quantization that compresses network models so as to accelerate DNN inference~\cite{hubara2016binarized,courbariaux2015binaryconnect,zhou2016doref,wang2019haq,yang2019quantization,sun2020ultra,choi2019accurate}.
Quantization of weights and activation can significantly reduce compute and storage overhead for efficient inference.
Researchers from IBM even achieve 4-bit training with low accuracy loss~\cite{sun2020ultra}.
However, radical quantization such as binarization of neural networks usually result in severe accuracy degradation~\cite{lin2017towards,hubara2016quantized,rastegari2016xnor}. 
To ensure the accuracy, researchers work towards pushing the bit width of the weight and activation to its limit without compromising the accuracy.
The state-of-the-art work found that 3-bit to 6-bit is an optimal bit width range for both weight and activation to ensure accuracy and performance~\cite{wang2019haq,yang2019quantization}.
However, even with low bit width weight and activation, weight memory access and costly arithmetic operations such as multiplication and addition are inevitable, which leads to considerable energy consumption and hardware footprint, hindering the deployment of edge AI. 
Therefore, in order to facilitate the deployment of neural network on embedded system and mobile applications, test-time inference without accuracy degradation that can totally eliminates the power hungry weight access and arithmetic operations is favored.

In this paper, we propose DTNN, which takes both model quantization and efficiency of operations into consideration.
Inspired by the tree structure of the dendrite in neural cells and the model quantization technique, DTNN achieves energy-efficient inference by replacing multiply-and-add (MAC) operations with table lookups.
Unlike traditional quantized model, in our DTNN design, weight could maintain the full bit width such as 32-bit or 64-bit, while only activation needs to be quantized, which makes DTNN not prone to accuracy drop.
With full bit width weight and low-bit activation, our DTNN achieves comparable accuracy with the full bit width baseline.
Meanwhile, by implementing nodes in tree structure with look-up tables, both costly weight access and compute-intensive arithmetic computations are eliminated in inference.

We conduct experiments from small to large DNN models ranging from LeNet-5, MobileNet to VGG, and ResNet.
In order to show the scalability of DTNN, we test it with different datasets, including MNIST, Cifar10/Cifar100, SVHN, and ImageNet.
According to the experiment results, DTNN achieves significant energy savings.
With negligible loss of accuracy, it achieves 19.4$\times$ and 64.9$\times$ energy improvement on ResNet-18 and VGG-11 with ImageNet, respectively.
To further validate the effectiveness of DTNN and compare with state-of-the-art low energy implementation for edge vision, we design and implement DTNN based MLP image classifiers using off-the-shelf FPGAs.
The results show that DTNN on the FPGA could achieve orders of magnitude better energy consumption and latency compared with the state-of-the-art low energy approaches reported that use ASIC chips.
Specifically, at the higher accuracy, our DTNN based image classifier running on FPGA consumed 1.13 nJ per image, or about 0.42\% of the 268 nJ per image reported for IBM's TrueNorth chip. In addition to such a high energy efficiency, DTNN on FPGA is five orders of magnitude faster than TrueNorth.

The contribution of this paper has four aspects:

\begin{enumerate}
  \item Proposing DTNN for efficient inference by using efficient lookup operations to replace costly multiplication and addition.
  \item Proposing activation-only quantization to support DTNN while ensuring accuracy.
  \item Leveraging tree structure to enable efficient lookup operation in DTNN.
  \item Achieving orders of magnitude improvement in energy consumption and latency compared with the state-of-the-art works.
\end{enumerate}

\section{Related Work}

In order to alleviate the computation and energy demand, model compression technique is proposed, which includes three major methods: pruning, quantization, and weight sharing~\cite{ding2019req,geng2019dataflow,8578919,hubara2016quantized,hubara2016binarized,courbariaux2015binaryconnect,zhou2016doref,wang2019haq,yang2019quantization,guo2016dynamic,han2015deep,liu2018frequency,xiao2017building}. 
In quantization, the precision of the weights and the activations are decreased.
In order to further reduce the computation cost of the models, more radical quantization methods are proposed such as BNN~\cite{hubara2016binarized}.
Binarization of both weights and activations reduces memory size and access, but results in severe accuracy degradation~\cite{lin2017towards}
Therefore, researchers are working towards pushing the bit width limits without compromising the accuracy.
In addition, many works propose model compression method specifically target to energy consumption of hardware~\cite{wang2019haq,yang2017designing}.
Wang et al.~\cite{wang2019haq} first manually set a constraint of energy consumption and then explore the quantization policy under this predefined constraint.
Yang et al.~\cite{yang2017designing} specifically prune the weights that are energy-consuming estimated by hardware.
The previous studies reduce the memory size and access and simplify the arithmetic operation by minimizing the bit width of weight and activations, but the weight access is still unavoidable and multiple bit width weight and activations incur multiplication and addition which limits the performance of the inference.

SNN approach is well known for its high energy efficiency due to its spike-based input/output, asynchronized computation/communication, and Non von Neumann architecture in hardware implementation~\cite{maass1997networks,Esser11441,pei2019towards}.
It consumes energy only when there are spike events, which results in high energy efficiency.
Dedicated chips are designed and fabricated for SNN such as TrueNorth and Tianjic~\cite{merolla2014million,pei2019towards}.
Using the TrueNorth platform, Esser et al. demonstrated the network that achieves the state-of-the-art energy efficiency for digit recognition~\cite{esser2015backpropagation}.

Unlike the methods mentioned above, our proposed DTNN is a hardware-friendly neural network that is inspired by the bio-structure of the dendrite in neural cells. By taking advantages of model quantization and using much more energy- and latency-efficient lookup operations to replace the multiplication and addition, our DTNN could achieve orders of magnitude better energy consumption and latency compared with traditional state-of-the-art efficient inference design.

\input{method}
\input{experiment}

\section{Conclusion}
Drawing inspiration from dendrites in neural cells, we proposed a novel Dtree neuron, and show how it can be used in deep neural networks with no loss of accuracy as compared to the traditional neuron. We then introduced a activation quantized version, and showed how a complete neural networks we called DTNN can be constructed and trained. By implementing the nodes in the DTNN model with lookup operations, both costly weight accesses and computationally expensive arithmetic computations are eliminated.
To validate its effectiveness, we conducted experiment on DTNN with various kinds of DNN models such as LeNet-5, MobileNet, VGG, and ResNet with different datasets, including MNIST, Cifar10, SVHN, and ImageNet.
DTNN achieved significant energy saving, specifically, 19.4$\times$ and 64.9$\times$ improvement on ResNet-18 and VGG-11 with ImageNet, respectively.
To further validate the effectiveness of DTNN and compare with state-of-the-art low energy SNN implementation for edge vision, we design and implement DTNN based MLP image classifiers using off-the-shelf FPGAs, and demonstrated several orders of magnitude improvement in energy consumption and speed.

{\small
\bibliographystyle{ieee_fullname}
\bibliography{egbib}
}

\end{document}

%% file: method.tex
\section{Method}\label{Motivation}
\subsection{Efficiency of different operations}
\label{section:LUT}

\begin{table}
\begin{center}
\begin{tabular}{|l|c|c|}
\hline
Operation & Consumed energy & Area cost\\
\hline\hline
32b FP Add & 0.9~$pJ$ & 4184~$um^2$\\
32b FP MUL & 3.7~$pJ$ & 7700~$um^2$\\
6-input LUT & 1.21$\times10^{-3}$~$pJ$ & 0.62~$um^2$\\
32b SRAM Read & 5~$pJ$ & NA\\
32b DRAM Read & 640~$pJ$ & NA\\
\hline
\end{tabular}
\end{center}
\caption{Energy of 32-bit floating point operations, 6-input, 1-output lookup operation and accesses to SRAM and DRAM for TSMC 45 nm technology node~\cite{hennessy2011computer,horowitz20141,abusultan2014look}}
\label{tab:op_efficiency}
\end{table}

\begin{figure}
  \centering
  \includegraphics[height=1.7in]{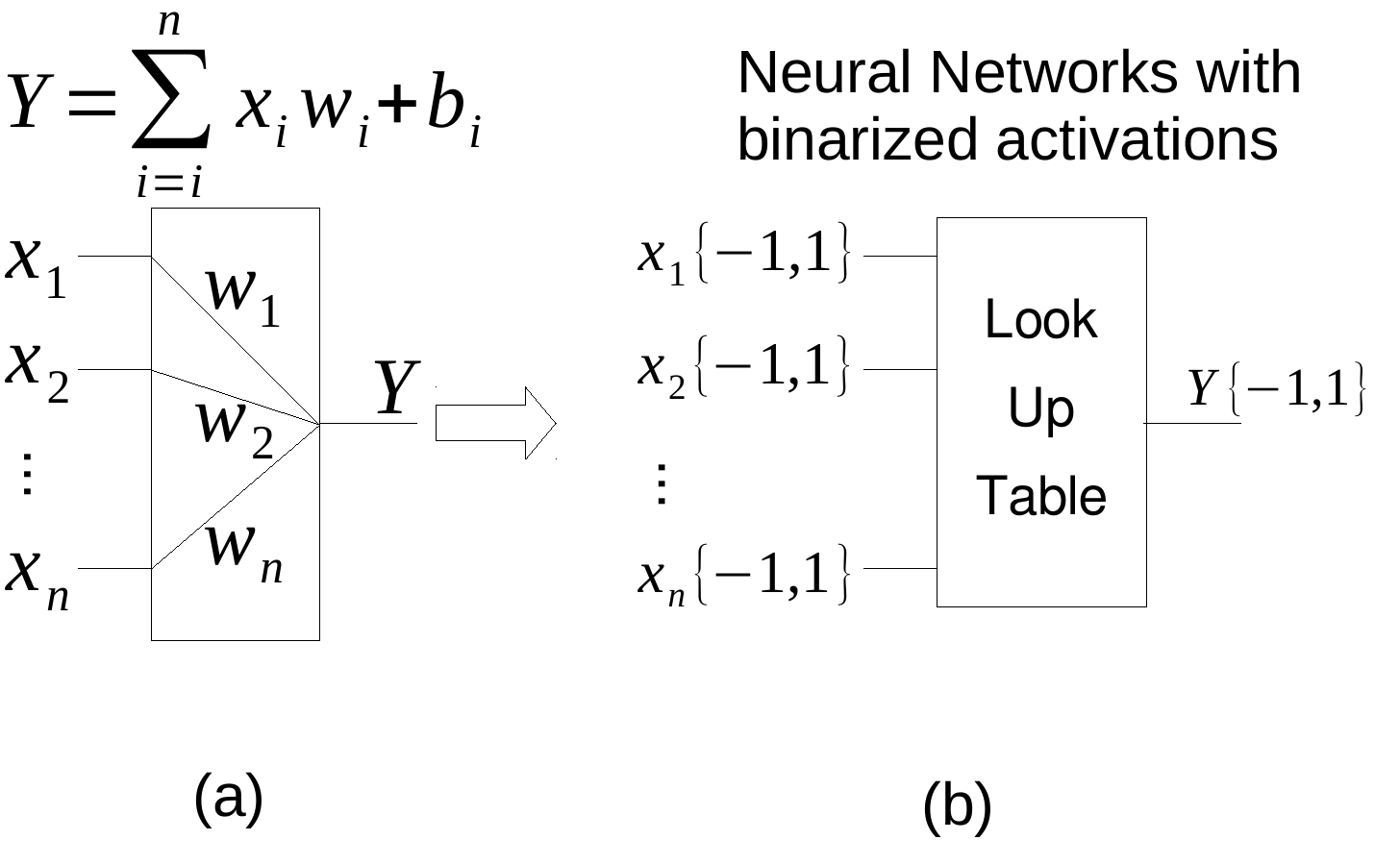}
  \caption{Realizing NN with binary activation by look-up table operation.(a) Core operation in ANN and BNN. (b) Look up table with binary input and output. }
  \label{fig:lut}
\end{figure}

As shown in Table~\ref{tab:op_efficiency}, among the different kinds of operations, memory accesses consume the most energy.
Floating point (FP) multiplication and addition operation consume nearly three orders of magnitude higher energy than a 6-input 1-output lookup operation.
Fig.~\ref{fig:lut}a shows a typical neural network layer in a neural network.
From a computation perspective, assuming that both weights and the activation are in full precision, the computation involved to compute $Y$ is the {\em multiplication-accumulation} (MAC) operation, which is the core operation in conventional artificial neural networks (ANN).
If both weight and activation are binarized, then it is the binarized neural networks~\cite{courbariaux2016binarized}, which brings benefit to hardware implementation but results in severe accuracy degradation. If we leave weight with full precision and make activations binarized, then the input and output of computation in the neural network layer are both constrained to only two possible values (e.g. -1 or 1).
With such a situation, the operations involved becomes {\em additive accumulation} (ACC). This is more efficient than MAC but is still costly because of the additions, and more importantly, of the memory accesses.

However, if we view the computation with binary digits inputs and output as a {\em lookup} operation as shown in Fig.~\ref{fig:lut}b, then the situation can be completely different.
A lookup table (LUT) of size $b 2^n$ can implement {\em any} Boolean function whose $n$ inputs and $b$ outputs are binary digits~\cite{mano}. It is the heart
of the logic in {\em field-programmable gate arrays} (FPGAs), which has high reconfigurablity and becomes a popular mean for verifying logic designs before ASIC fabrication.
{\em The key idea of DTNN is to use efficient lookup operation to replace costly floating point or fixed point multiplication and addition in inferences.}
We use a neuron unit with six binary digits inputs and a single binary digit output as an example to show how efficient it can be.
Namely, in this example, we use a lookup table of size $b 2^n$, where $b$ and $n$ equal to $1$ and $6$, respectively.

Table~\ref{cost_table} compares the cost of the 6-input neuron implementation with 32-bit FP addition versus implementation with 6-input LUT in terms of energy, delay, and area~\cite{nathan2014recycled,abusultan2014look}.
\begin{table}
  \caption{Costs of the 6-input neuron implementation with 32-bit float point addition VS 6-input LUT in terms of energy, delay, and area.}
  \label{cost_table}
  \centering
  \begin{tabular}{|l|c|c|}
  \hline
    Name     & \tabincell{c}{32-bit FP addition\\(Normalized with LUT)} & 6-input LUT \\
     \hline \hline
    Energy & 6.2 $pJ$~~~~~~~~~(5124x)      & 1.21 $fJ$    \\
    Delay  & 6.65 $ns$~~~~~~~~(168x)    & 39.5$ps$    \\
    Area   & 1201 $\mu m^2$~~~(1937x)   & 0.62 $\mu m^2$ \\
  \hline
  \end{tabular}
\end{table}
From the table, we can see that both the delay and energy consumption of the FP addition sechme are orders of magnitude higher than that of the 6-input LUT.
More importantly, weights -- the operands of the floating point additions -- have to be accessed from some form of centralized SRAM or DRAM, which would incur a significant energy overhead. This is not so with the LUT as they are already implicitly encoded in the LUT.

As stated above, we can use a lookup table of size $b 2^n$ to implement any Boolean function with $n$-bit inputs and $b$-bit outputs. 
However, to use LUT as a fundamental building block, we need to solve a major obstacle: large fan-ins are required to implement the layers of a neural networks. 
The typical number of fan-ins for each layer of neural networks varies from tens to several thousands.
An $n$-input, $b$-bits output LUT will require $b 2^n$ bits of storage. Clearly, this is infeasible if $n$ is in the thousands.

\begin{figure*}
  \centering
  \includegraphics[height=2.0in]{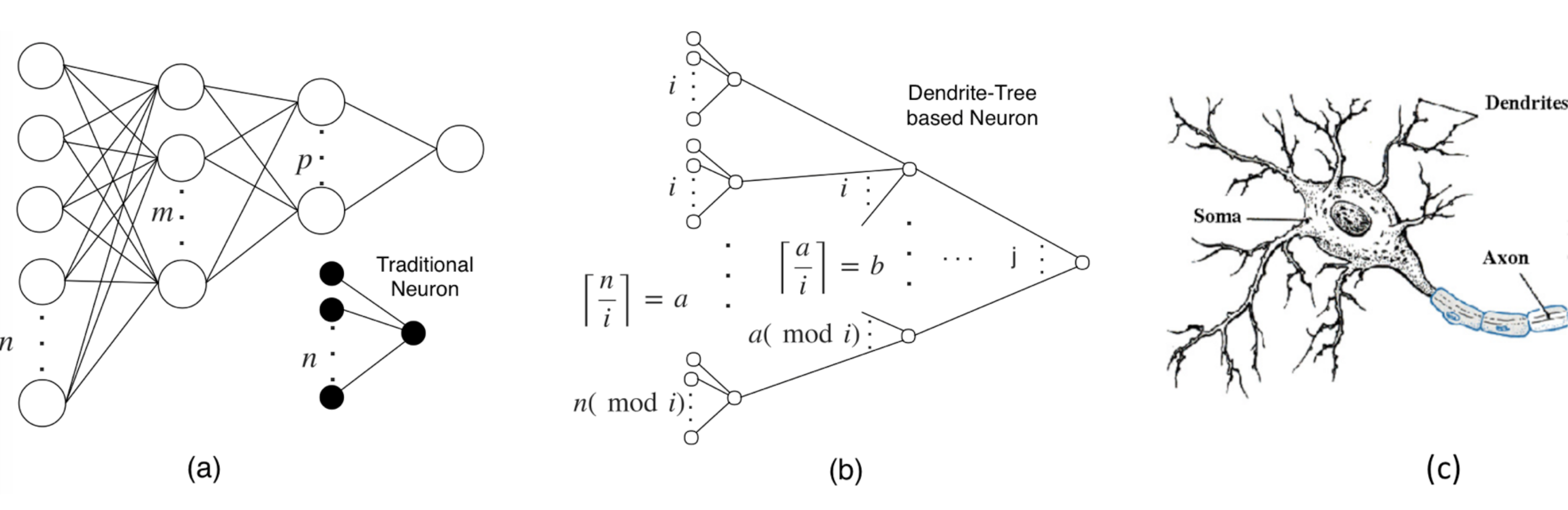}
  \caption{(a) MLP with two hidden layers built with the traditional neurons. (b) The schematic of a D-neuron. (c) The biomorph of biological neuron.}
  \label{neuron}
\end{figure*} 

To tackle the fan-in problem, we drew inspiration from the tree structure of dendrites in neural cells (Fig~\ref{neuron}c) as well as quantization technique, and propose a dendrite-tree inspired neural network~(DTNN) with activation quantization.
In neural cells, dendrites are connected with thousands of axons from previous neural cells.
Dendrites are organized in a tree structure to handle the thousands of connections.
We shall now show how this approach will solve the fan-in issue.

\subsection{Dendrite-tree inspired neural networks}
In this section, we shall describe the structure of our proposed DTNN in detail.
One of the key advantages of DTNN is that many of the state-of-the-art techniques, such as model compression with pruning, can be applied to DTNN.

\subsubsection{Dtree neuron}\label{sec:D-neuron}
In neural network models, the basic building block is a neuron.
In this section, we shall describe a general dendrite-tree inspired neuron unit we call a {\em Dtree neuron} that in theory can be used to approximate any neural network model. A quantized version of this is what is used to build DTNN.

For the ease of demonstration, we use multilayer perceptron (MLP) to demonstrate the Dtree neuron design, which could be generalized to any neuron based neural networks such as the CNN. In a multilayer perceptron (MLP), multiple neurons between layers are fully connected with one another.
Fig.~\ref{neuron}a shows a three-layer MLP with two hidden layers.
Each node in the MLP is a neuron that employs a nonlinear activation function. It takes the outputs of neurons from the previous layer as inputs, and generates an output that serves as an input to neurons from the next layer.
The function realized by the neuron with $n$ inputs is shown with Eqn.~\ref{eqn:tradition_neuron}:
\begin{equation}\label{eqn:tradition_neuron}
  y = f(\displaystyle\sum_{i=0}^{n-1} w_ix_i + b),
\end{equation}
where $x_i$ and $y$ represent the input and output of the neuron, and $w_i$ and $b$ are the weights and bias respectively.
The function $f$ is a nonlinear activation function, such as rectified linear unit (ReLU) or a hyperbolic tangent function (Tanh).
A traditional neuron in the first hidden layer with $n$ inputs is shown as dark circles in Fig.~\ref{neuron}a.


The Dtree neuron, designed to solve the large fan-in problem is shown in Fig.~\ref{neuron}b.
It is inspired by a biological neuron that utilizes a ``\textit{tree}'' structure, the ``\textit{dendrite}'', to handle the large amount of inputs. Its biomorph is shown in Fig.~\ref{neuron}c.
Instead of connecting all the inputs to the neuron directly, the Dtree neuron uses a tree with intermediate neurons (inner neurons) that are organized as {\em inner layers}. Thus, Fig.~\ref{neuron}b is the counterpart of the traditional neuron in Fig.~\ref{neuron}a, both
having $n$ inputs and one output.
The fan-in of each node in the tree structure can vary, but for simplicity, we constrain the maximum fan-in number to be $i$.
With this configuration, for a neuron with $n$ inputs, the number of inner layers is $\left \lfloor{\log_i n}\right \rfloor = L$, excluding the input layer.
The number of nodes for the first inner layer is $\left \lceil{\frac{n}{i}}\right \rceil = a$, and each of them has a fan-in of $i$, except for the last node, which has $n~(mod~i)$ inputs.
Subsequently, the outputs of first inner layer act as inputs to the next inner layer and so on.
The function realized by the dendrite-tree inspired neuron is shown in Eqn.~\ref{eqn:dt_neuron},
\begin{equation}\label{eqn:dt_neuron}
\begin{split}
  y = &f_{(L-1)}(\displaystyle\sum_{i}^{} w^{(L-1)}_i\cdot f_{(L-2)}(\displaystyle\sum_{j}^{} w^{(L-2)}_j \cdot\cdot\cdot\\ &f_0(\displaystyle\sum_{k}^{} w^{0}_kx_k + b_0) \cdot\cdot\cdot + b_{(L-2)}) + b_{(L-1)}),
\end{split}
\end{equation}
where $w^{L-1}$, $w^{L-2}$, and $w^{0}$ are the weights in the inner layers, from the last inner layer all the way to the first inner layer.
Similarly, $b_{L-1}$, $b_{L-2}$, and $b_0$ are the biases of the inner layers, and $f_{L-1}$, $f_{L-2}$, and $f_0$ are activation functions of the inner layers, from the last inner layer to the first inner layer. Note that the activation functions can be different for each inner layer.

\begin{figure*}
  \centering
  \includegraphics[width= 0.85\textwidth]{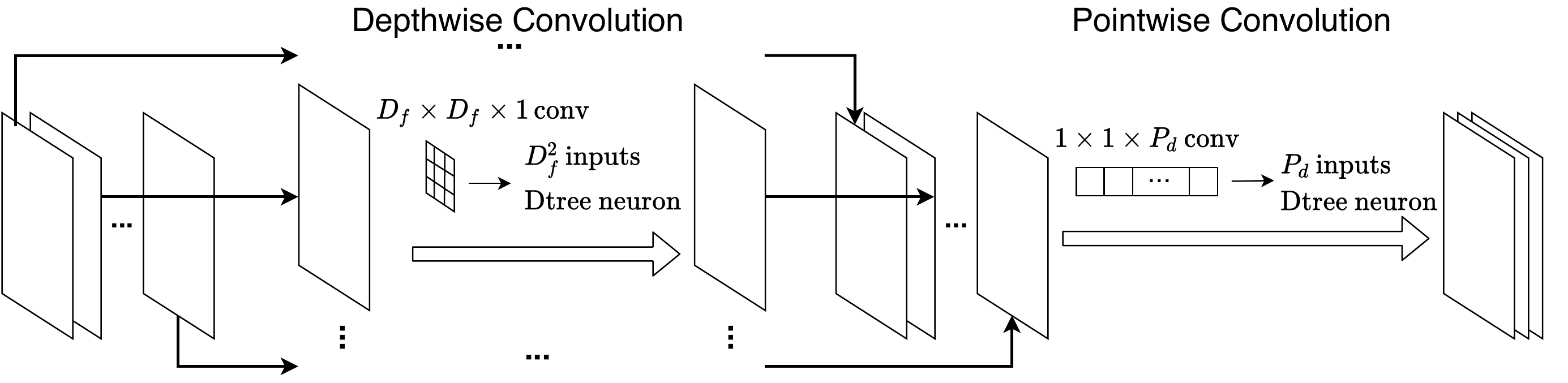}
  \caption{Depthwise and poinwise convolutional layers realized by Dtree neuron}
  \label{fig:depthwise_pointwise}
\end{figure*}

A neural network can be viewed as a function that maps input features to outputs.
The following lemma gives us the basis for approximating a traditional NN using Dtree neurons.
\begin{lemma}\label{lemma}
Given a function that represents a traditional neuron with weights and bias, $y = f(\displaystyle\sum_{i}^{} w_ix_i + b),$ where $w_i, b \in \mathbb{R}$ and $f$ is an activation function, there exists weights, biases, and activation functions for a Dtree neuron that computes the same function.
\end{lemma}

\begin{proof}
Assume the traditional neuron has $n$ inputs, and the function it can realize is of the form in Eqn.~\ref{eqn:tradition_neuron}.
Rewriting Eqn.~\ref{eqn:tradition_neuron}, we get the equation shown below:
\begin{equation}\label{eqn:proof_tradition}
\begin{split}
  &y = f(W + b), \\ 
  &W = \displaystyle\sum_{i}^{n}w_ix_i.
\end{split}
\end{equation}
For an $n$ input Dtree neuron with $L$ inner layers, all the functions that it can represent are shown by Eqn.~\ref{eqn:dt_neuron}.
We perform the following transformations on it: 1) set $f_{L-1}$ to be the same function as $f$ in Eqn.~\ref{eqn:proof_tradition}, and set $f_{L-2}$ to $f_0$ to be the identity function, 2) set $b_{L-1}$ to $b$ in Eqn.~\ref{eqn:proof_tradition}, and $b_{L-2}$ to $b_0$ to zero.
Thus, the function in Eqn.~\ref{eqn:dt_neuron} is transformed to Eqn.~\ref{eqn:proof_dt_neuron}, where we used the indices $p$, $q$, $r$, $s$ without loss of generality.
\begin{equation}\label{eqn:proof_dt_neuron}
\begin{split}
  &y = f(W_{D} + b), \\
  &W_{D} = \displaystyle\sum_{p}^{} w^{(L-1)}_p\cdot(\displaystyle\sum_{q}^{} w^{(L-2)}_q \cdot\cdot\cdot 
  (\displaystyle\sum_{r}^{} w^{1}_r(\displaystyle\sum_{s}^{} w^{0}_sx_s)))
\end{split}
\end{equation}
Comparing Eqn.~\ref{eqn:proof_tradition} and Eqn.~\ref{eqn:proof_dt_neuron}, as long as we can show that $W_{D}$ and $W$ are identical, then the lemma holds.
From Eqn.~\ref{eqn:proof_tradition}, we can see that $W$ is the sum of each input $x_i$ multiplied by its weight $w_i$
Therefore, for any input $x_i$, its contribution to $W$ is $w_i x_i$.
Similarly, according to Eqn.~\ref{eqn:proof_dt_neuron}, for any input $x_s$, its contribution to $W_{D}$ is $w^{L-1}_pw^{L-2}_q\cdot\cdot\cdot w^{1}_rw^{0}_sx_s$
Therefore, if $x_i$ and $x_s$ are corresponding inputs in traditional neuron and Dtree neuron, then as long as $w_i$ equals to $w^{L-1}_p w^{L-2}_q\cdot\cdot\cdot w^{1}_rw^{0}_s$, the contribution to $W$ and $W_{D}$ is the same for the two inputs.
Since $w^0_s$ is only relevant to $x_s$, as long as $w^0_s$ equals  $\frac{w_i}{w^{L-1}_pw^{L-2}_q\cdot\cdot\cdot w^{1}_r}$, then for $x_i$ and $x_s$, their contribution to $W$ and $W_{D}$ is identical.
Similarly, for the remaining inputs of the neuron, the contribution of each of the inputs in Dtree neuron to $W_{D}$ can be made identical as its counterpart's contribution in the traditional neuron to $W$, and hence $W_{D}$ is made identical to $W$.
\end{proof}

Lemma~\ref{lemma} shows that theoretically, a Dtree neuron would have no lower accuracy than the traditional neuron if there is free choice in selecting the activation functions of the inner layers.
For simplicity, we consider two types of Dtree neuron in this work.
According to the lemma, as long as the activation functions of the inner layers are identity functions and the weights of the inner layers are $1$, the Dtree neuron would act like a traditional neuron. We call this type of neuron the {\em Type-I Dtree neuron}.
In the {\em Type-II Dtree neuron} the activation function of the inner layers are the same as the traditional neuron such as ReLU or Tanh, and the weights of the inner layers are trainable parameters.
The Type-II Dtree neuron has more trainable weights compared to a traditional neuron. 
For the $n$-input traditional neuron, the number of weights is $n$, whereas for the $n$-input Type-II Dtree neuron, with $m$ as the fan-in number of each node in the tree structure, the number of weights is denoted by $N_{D}$, where
\begin{equation}\label{eqn:number_additional}
N_{D} = n+n\frac{(1-(\frac{1}{m})^{\mbox{log}_m n})}{m-1} 
\end{equation}

Although the weights of DTNN are in full precision, the quantization of activation would still results in the information loss.
Fortunately, according to previous work, activations that are in 3-bit to 5-bit width range are sufficient to ensure accuracy in actual networks~\cite{wang2019haq}.

\subsubsection{DTNN based fully connected and convolutional layer}

In order to realize DTNN based neural networks, we just need to replace conventional neuron in NN by Dtree neuron.
For fully connected layer, it is quite straight forward.
For convolutional layer, we use MobileNet as an example to demonstrate how to build a DTNN based CNN.
MobileNet~\cite{howard2017mobilenets} is inspired from the depthwise separable convolutions~\cite{chollet2017xception}.
The regular convolution operations are replaced by the pointwise and depthwise convolutions, which could reduce large amount of MAC operations.
As shown in Fig.~\ref{fig:depthwise_pointwise}, MobileNets stack multiple “depthwise – pointwise” blocks repeatedly.
By replacing conventional neuron in MobileNet with Dtree neuron, we get DTNN based MobileNet, which is also shown in Fig.~\ref{fig:depthwise_pointwise}.

\subsection{Input pre-processing}\label{sec:input_prep}
As DTNN belongs to quantized neural networks, there are two methods to feed in the input data.
We use an image classification application as an example. 
In Method 1, a normal convolution layer is employed that takes as its input the high precision image.
With this method, the first layer have to be implemented using traditional floating point processing unit. which incurs extra pre-processing cost in terms of energy and hardware resources.

In Method 2, in order to achieve ultra low energy consumption, we directly binarize the input pixel to binary values, which involves negligible pre-processing effort. 
Given an input $x$, and output $y$, is simply
\begin{equation}\label{eq:activation}
    y=\left\{
        \begin{array}{rcl}
        1     &      & {x      >      \mbox{threshold}},\\
        0     &      & {\mbox{otherwise.}}
        \end{array} \right.
\end{equation}
The threshold can be learnt or be preset. Details are given in Section~\ref{sec:MNIST_network}.
Please note this radical method does not work well for large models with large datasets due to considerable information loss. 

%% file: experiment.tex
\section{Experiment}
\textbf{Datasets and models.} In this section, we will present experimental results of DTNN on various kinds of DNN models such as LeNet-5, MobileNet,  VGG, and ResNet with different datasets, including MNIST, Cifar10, SVHN, and ImageNet.
LeNet-5 is a simple neural network with only two neural layers~\cite{lecun1998gradient}.
MobileNet is a compact model with depth-wise and point-wise convolution, which is designed specifically for mobile vision application~\cite{howard2017mobilenets}.
VGG is a standard model with conventional convolution and fully connected layers~\cite{simonyan2014very}.
ResNet is a classical neural network used as a backbone for many current computer vision tasks~\cite{he2016deep}.
It is featured with residual connections to prevent gradient from vanishing.

\textbf{Training settings.} We conduct experiments using quantization-aware training in Pytorch~\cite{paszke2019pytorch}.
For the MNIST and CIFAR10/Cifar100 dataset, the training takes a total of 220 epochs with a batch size of 512. 
The initial learning rate is set to 0.004 and then divided by 2 at every 30 epochs.
An Adam solver is adopted with betas=(0.9, 0.999), eps=$10^{-8}$, and a weight decay of $10^{-4}$.
For the SVHN dataset, the training takes a total of 70 epochs, and the rest of training settings are the same with MNIST and Cifar10/Cifar100.
For ImageNet, the training takes a total of 90 epochs with a batch size of 64.
The initial learning rate is set to 0.1 and then divided by 10 at every 30 epochs.
A SGD solver with momentum set at 0.9, and a weight decay of $10^{-4}$ was used.

\subsection{DTNN results on a small model, LeNet-5, with MNIST}
\label{sec:LeNet-5_MNIST}

\begin{table}
   \small
  \caption{Accuracy and energy (uJ) numbers of DTNN based LeNet-5 with different bit width of activation implemented by lookup operations and ANN based LeNet-5 by conventional multiplication and addition for MNIST.}

  \centering
 
  \begin{tabular}{lrrrr}

    {\bf Metrics}        & {\bf 4-bit} & {\bf 5-bit} & {\bf 6-bit} & {\bf 32FP}\\
    \midrule
    Acc.         & 99.14\%       & 99.32\% & 99.18\%  & 99.10\%  \\
    Energy            & 0.5411        & 0.5416  & 0.5441 & $40.73$ \\
    \bottomrule
  \end{tabular}
  \label{tab:lenet-5}
\end{table}
Table.~\ref{tab:lenet-5} shows Accuracy and energy numbers of DTNN based LeNet-5 with different bit width of activation implemented by lookup operations and ANN based LeNet-5 by conventional FP multiplication and addition for MNIST.
According to the table, all designs with quantized activation has no accuracy loss compared with its ANN counterpart.
Up to 75.3$\times$ energy improvement is achieved.
The reason why the consumed energy is similar for 4-bit to 6-bit activation implementation is because we implement the first layer in DTNN with traditional FP multiplication and addition, which consumes the most energy 0.54 $uJ$.

\subsection{DTNN results on large models, ResNet-18 and VGG-11, with different datasets}

\begin{table}
  \small
  
  \caption{Accuracy of DTNN based ResNet-18 with different bit width of activation implemented by lookup operation and its ANN counterpart by conventional FP multiplication and addition.}
  \label{Tab:acc-resnet}
  \centering
  \begin{tabular}{lrrrrr}

    {\bf Dataset}        & {\bf 3-bit} & {\bf 4-bit} & {\bf 5-bit} & {\bf 6-bit} & {\bf 32FP}\\
   \midrule
    Cifar10    & 88.8\%     & 92.4\%         & 92.7\%  & 93.1\% & 93.0\% \\
    Cifar100   & 65.5\%     & 71.3\%         & 72.3\%  & 72.9\% & 73.0\% \\
    SVHN   & 92.8\%    & 95.6\%   & 96.3\% & 96.3\% & 96.2\%\\
    ImageNet  & 57.9\%    & 67.1\%   & 69.6\% & 70.3\% & 69.8\%\\
 \bottomrule
  \end{tabular}
 \label{Tab:acc-resnet}
\end{table}

Table.~\ref{Tab:acc-resnet} shows accuracy of DTNN based ResNet-18 with different bit width of activation implemented by lookup operation and its ANN counterpart by conventional FP multiplication and addition.
In order to show the scalability of DTNN, we test it with both small dataset including cifar10/SVNH and large datasets, i.e., cifar100/ImageNet.
According to the table we can see that large dataset are more prone to accuracy drop due to the decrease of activation bit width, but both of them has negligible loss of accuracy with 4-bit to 5-bit activation and above.

\begin{table}
  \small
  \caption{Energy (uJ) of DTNN based ResNet-18 implemented by lookup operation and ANN based  ResNet-18 by multiplication and addition in inference for one image.}

  \centering
 
  \begin{tabular}{lrrrrr}

    {\bf Dataset}        & {\bf 4-bit} & {\bf 5-bit} & {\bf 6-bit} & {\bf 32FP}\\
    \midrule
    Cifar10/SVHN         & 31.46        & 125 & 569.2  & $1.08 \times 10^4$  \\
    Cifar100            & 31.46        & 125  & 569.3 & $1.09 \times 10^4$ \\
    ImageNet     & 613   & 900 & $2.26 \times 10^3$ & $1.75 \times 10^4$\\
    \bottomrule
  \end{tabular}
 \label{Tab:energy-resnet} 
\end{table}

Table.~\ref{Tab:energy-resnet} shows energy of DTNN based RseNet-18 implemented by lookup operation and ANN based RseNet-18 by FP multiplication and addition in inference for one image.
According to the table, DTNN achieves at least 19.4$\times$ energy saving on ImageNet with negligible loss of accuracy.
For small dataset, DTNN achieves 87.2$\times$ energy saving on cifar10.

\begin{table}

  \small
  \caption{Accuracy and energy (uJ) numbers of DTNN based VGG-11 with different bit width of activation implemented by lookup operations and ANN based VGG-11 by conventional FP multiplication and addition for ImageNet.}

  \centering
 
  \begin{tabular}{lrrrr}

    {\bf Metrics}        & {\bf 4-bit} & {\bf 5-bit} & {\bf 6-bit} & {\bf 32FP}\\
    \midrule
    Acc.         & 64.5\%        & 69.5\% & 71.0\%  & 70.4\%  \\
    Energy            & 686.9        & $1.85 \times 10^3$  & $7.38 \times 10^3$ & $1.20 \times 10^5$ \\
    \bottomrule
  \end{tabular}
  \label{Tab:vgg}
\end{table}

Table.~\ref{Tab:vgg} shows results on another standard DNN model, VGG-11, with ImageNet.
16.3$\times$ to 64.9$\times$ energy saving is achieved by DTNN with similar accuracy compared with its ANN counterpart.

\subsection{DTNN results on a compact model, MobileNet, with different datasets}

\begin{table}
  \small
  
  \caption{Accuracy of DTNN based MobileNet with different bit width of activation implemented by lookup operation and its ANN counterpart by conventional FP multiplication and addition.}

  \centering
  \begin{tabular}{lrrrrr}

    {\bf Dataset}      & {\bf 4-bit} & {\bf 5-bit} & {\bf 6-bit} & {\bf 32FP}\\
     \midrule
    Cifar10         & 89.0\%         & 90.4\%  & 90.3\% & 90.9\% \\
    Cifar100     & 61.4\%         & 64.3\%  & 64.7\% & 65.0\% \\
    SVHN      & 94.0\%   & 95.6\% & 96.1\% & 96.0\%\\
    ImageNet     & 53.3\%   & 64.15\% & 67.43\% & 68.9\%\\
    \bottomrule
  \end{tabular}
  \label{tab:acc-mobilenet}
\end{table}
Table.~\ref{tab:acc-mobilenet} shows accuracy of DTNN based MobileNet with different bit width of activation implemented by lookup operation and its ANN counterpart by conventional FP multiplication and addition.
According to the table, compared with standard DNN models, MobileNet, as a compact model for mobile application is more prone to accuracy drop due to the decrease of activation bit width.

\begin{table}
  \small
  \caption{Energy (uJ) of DTNN based MobileNet implemented by lookup operation and ANN based  MobileNet by multiplication and addition in inference for one image.}

  \centering
 
  \begin{tabular}{lrrrrr}

    {\bf Dataset}        & {\bf 4-bit} & {\bf 5-bit} & {\bf 6-bit} & {\bf 32FP}\\
    \midrule
    Cifar10/SVHN         & 5.79         & 12.8 & 46.3 & $2.27 \times 10^3$  \\
    Cifar100            & 5.80         & 12.9  & 46.4 & $2.33 \times 10^3$ \\
    ImageNet     & 71.1   & $157$ & $568$ & $5.33 \times 10^3$\\
    \bottomrule
  \end{tabular}
  \label{Tab:energy-mobilenet}
\end{table}

Table.~\ref{Tab:energy-mobilenet} shows the energy of DTNN based MobileNet implemented by lookup operation and ANN based MobileNet by multiplication and addition in inference for one image.
We can see that with negligible loss of accuracy, DTNN could achieve 9.4$\times$ and up to 181$\times$ energy saving on MobileNet with large dataset, i.e., ImageNet dataset and small dataset, respectively.
This result is inline with our expectation because MobileNet is a compact model, which is not as capable as other standard models such as ResNet and VGG. 
Therefore, it is prone to accuracy drop in model quantization for the large dataset, which limits the power saving with DTNN.

\subsection{Comparison with the state-of-the-art}

 \begin{figure}
  \centering
  \includegraphics[height=1.8in]{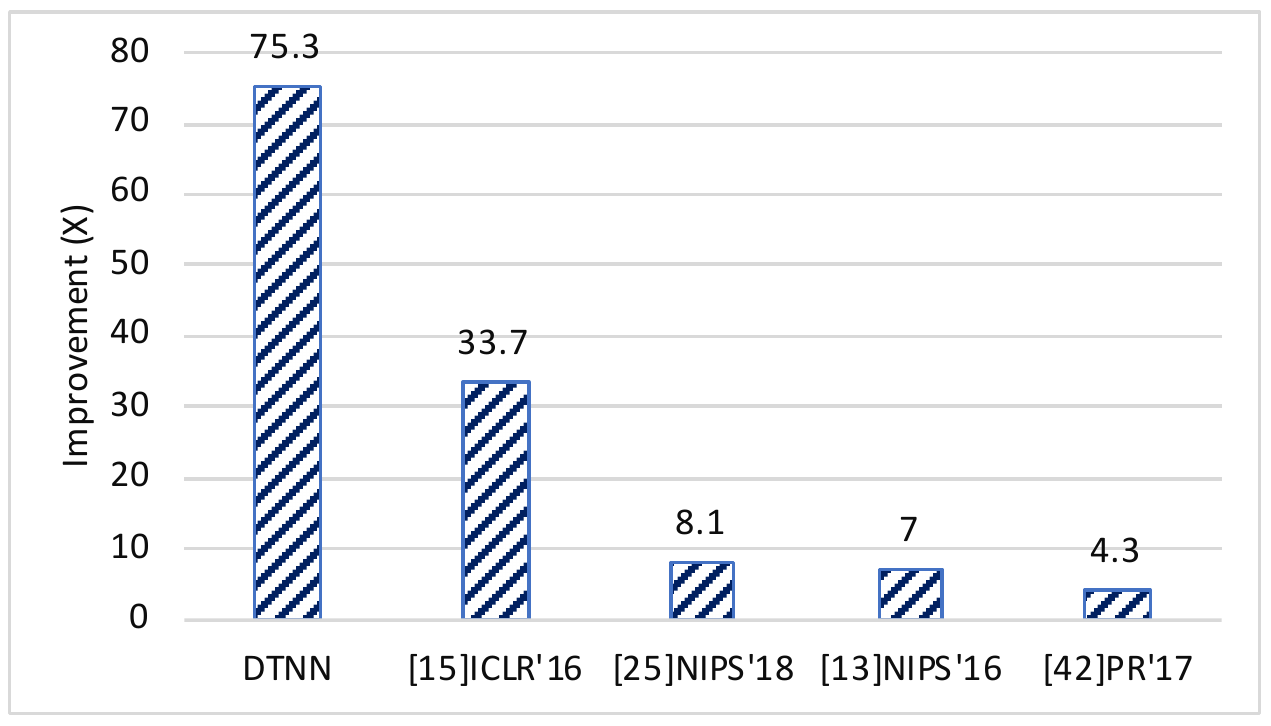}
  \caption{Energy improvement achieved by DTNN with LeNet-5 on MNIST compared with the state-of-the-art model compression methods~\cite{han2015deep,guo2016dynamic,xiao2017building,liu2018frequency}}
  \label{fig:comparison_MNIST}
\end{figure}

Fig.~\ref{fig:comparison_MNIST} shows the energy improvement achieved by DTNN with LeNet-5 on MNIST compared with the state-of-the-art model compression methods for efficient inference~\cite{han2015deep,guo2016dynamic,xiao2017building,liu2018frequency}.
From the figure, we can see that DTNN achieves 75.3$\times$ energy improvement, which is much larger than that of other SOTA results.

\begin{table}
   \small
  \caption{Comparison of accuracy and energy improvement between DTNN based MobileNet and SOTA model compression method on ImageNet.}

  \centering
 
  \begin{tabular}{lrrr}

    {\bf Metrics}        & {\bf 5-bit DTNN} & {\bf 6-bit DTNN} & {\bf \cite{wang2019haq}CVPR'19} \\
    \midrule
    Acc.         & 64.15\%       & 67.43\% & 64.8\%    \\
    \tabincell{l}{Energy \\ Improv.}            & 40$\times$        & 9.4$\times$  & 11.2$\times$  \\
    \bottomrule
  \end{tabular}
  \label{tab:comparison_MobileNet}
\end{table}

Table.~\ref{tab:comparison_MobileNet} shows the comparison of accuracy and energy improvement between DTNN based MobileNet and SOTA model compression method with MobileNet on ImageNet.
According to the table, with similar accuracy, DTNN achieved 40$\times$ energy saving, which is much larger than the SOTA 11.2$\times$. 

\subsection{Comparison with renowned low energy SNN design implemented on ASIC}
\label{sec:MNIST_network}

In order to further validate the effectiveness of DTNN, in this section, we design and implement DTNN based image classifiers using off-the-shelf FPGAs and compare it with state-of-the-art low energy ASIC implementation for edge vision. Please note that the work with SOTA energy results usually focus on the energy efficiency and don't have SOTA accuracy.
In order to conduct fair comparison, we design MLP classifiers, which as similar accuracy range as the SOTA energy-efficient work and compare with them on the energy efficiency. In this experiment we use Type-II Dtree neuron and input pre-processing Method 2 described in Section~\ref{sec:input_prep} to binarize the input image for higher energy performance.

\begin{table*}[!htb]
  \caption{Comparison between DTNN and other state-of-the-art low energy classifiers on MNIST in terms of accuracy, energy per image, and images per second.}
  \label{demo_table}
  \centering
  \begin{tabular}{llrrrr}
    {\bf Approach}     & {\bf Accuracy}    & {\bf Energy/image} & {\bf Images/second} & {\bf Tech.} & {\bf \# 6-LUT}\\
    \midrule
    SNNwt~\cite{du2015neuromorphic}    & 91.8\%    & 2.15 x $10^{5}$nJ     & NA         & \phantom{1}65nm  & -\\
    TrueNorth~\cite{esser2015backpropagation}    & 92.7\%    & 268nJ     & 1 x $10^{3}$         & \phantom{1}28nm  & -\\
    DTNN MLP-1 & 92.8\%    & 1.13nJ    & 2 x $10^{8}$   & \phantom{1}28nm & 7.90K\\
    \midrule
    TrueNorth~\cite{esser2015backpropagation}    & 95.0\%    & 4000nJ    & 1 x $10^{3}$         & \phantom{1}28nm  & -\\
    Shenjing~\cite{wang2019shenjing} & 96.11\%   & 2020nJ     & 5.95 x $10^{4}$   & \phantom{1}28nm  & -\\
    DTNN MLP-2  & 96.55\%   & 21.97 nJ     & 1.25 x $10^{8}$   & \phantom{1}28nm  & 172.79K\\
    Tianjic~\cite{pei2019towards} & 96.59\%   & 3.8 x $10^{4}$nJ     & 40   & 28nm  & -\\
    SNN(FPGA)~\cite{han2020hardware} & 97.06\%   & 2.96 x $10^{6}$nJ     & 161   & \phantom{1}28nm  & NA*\\
    \bottomrule
    \multicolumn{6}{p{1.7\columnwidth}}{\vspace{1mm} *The use of LUT is different and not comparable. As with most other FPGA implementations, they used LUTs to implement computing logic, while storing weights in block RAMs.}
  \end{tabular}
  
\end{table*}

\textbf{Classifier architectures}
We adopt MLP in this experiment, because MNIST is a small dataset, which is easy to reach above 99\% accuracy with small CNNs as shown in Section~\ref{sec:LeNet-5_MNIST}.
Using CNN cannot get similar accuracy with SOTA energy efficient SNN works with ASIC chips.
We built two classifiers based on structures of MLP (DTNN MLP-1/2) for MNIST, which are \texttt{28x28x1-10} and \texttt{28x28x1-216-10}, respectively.
We use Method 2 in Section~\ref{sec:input_prep} to generate the binary input from the dataset.
As Method 2 directly binarizes pixels in input images with a single threshold, which would invariably lead to information loss, we build the classifier by utilizing an ensemble of 5 MLPs, where each individual MLP receives as input the 784 pixels of an MNIST digit which has been binarized by a separate threshold value that is unique to that MLP, this allows the digit to be fed in at several different threshold values which mitigates the information loss that occurs due to binarization (this is analogous to HDR algorithms that capture more information about a scene by snapping multiple shots at different exposure levels). Although these threshold values could in principle be learnt, for simplicity we used five evenly spaced threshold values.
For pixel range from 0 to 1, we set five thresholds as 0.2, 0.4, 0.5, 0.6, 0.8, respectively.
An ensemble of 5 MLPs yields 5 output values per digit class, so in both our classifiers we apply a binarized linear function to each of the 5 values to reduce the output to a single binary value per digit class. The weights of the binarized linear function are learnt when training the ensemble of MLPs, and it, too, is constructed from a Dtree neuron consisting of fixed 6-input inner neurons.

\textbf{Training of DTNN based MLP-1/2}
DTNN MLP-1/2 were trained using the Adam optimizer, minimizing the cross-entropy loss. The parameters of the Adam optimizer are shown as follows (learning rate = $0.001$, beta = $0.9$, beta2 = $0.999$, epsilon = $10^{-8}$).
DTNN MLP-1 was trained with a batch size of 256 for up to $20,000$ iterations. DTNN MLP-2 was trained with a batch size of 1024 for up to $500,000$ iterations.

\textbf{Results comparisons} 
We implemented the DTNN MLP-1 design on AVNET Zedboard with Xilinx Zynq 7020 (XC7Z020-CLG484-1) FPGA, and the DTNN MLP-2 design on Alpha Data ADM-PCIE-7V3 with Xilinx Virtex-7 (XC7VX690T-2) FPGA.
In general, our DTNN architecture is not limited to implementation on commercial FPGAs only. It can also be implemented on ASICs. 
However, even with generalized commercial FPGAs, we obtained much better results than other low power approaches on dedicated ASIC chips.
Table~\ref{demo_table} shows different metrics including the accuracy, energy consumption per image, and images processed per second.

DTNN MLP-1, with no hidden neurons, achieves 92.8\% accuracy with 1.13~nJ consumption per image, which is 0.42\% of the 268~nJ per image for the TrueNorth  92.7\% accuracy NN.
The second DTNN MLP-2, which has a hidden layer containing 216 neurons. achieved a higher accuracy of 96.55\%, and an energy consumption of 21.97~nJ at a rate of 125M images per second. In comparison, a higher accuracy version on TrueNorth consumes 4000~nJ and runs at 1K images per second with a slightly lower accuracy.